\documentclass{article}

\usepackage{microtype}
\usepackage{graphicx}
\usepackage{subfigure}
\usepackage{booktabs} 

\usepackage{hyperref}

\usepackage{amsfonts}
\usepackage{amsthm}
\usepackage{amsmath}

\newcommand{\eps}{\epsilon}

\newcommand{\mc}{\mathcal}
\newcommand{\mb}{\mathbb}

\newcommand{\indep}{\perp \!\!\! \perp}

\newcommand{\pa}{\operatorname{pa}}
\newcommand{\anc}{\operatorname{anc}}
\newcommand{\support}{\operatorname{supp}}

\newtheorem{assumption}{Assumption}
\newtheorem{definition}{Definition}
\newtheorem{proposition}{Proposition}
\newtheorem{lemma}{Lemma}
\newtheorem{theorem}{Theorem}
\newtheorem{corollary}{Corollary}

 \usepackage[accepted]{icml2021}

\icmltitlerunning{On the Sample Complexity of Causal Discovery and the Value of Domain Expertise}

\begin{document}

\twocolumn[
\icmltitle{On the Sample Complexity of Causal Discovery and \\
the Value of Domain Expertise}



\icmlsetsymbol{equal}{*}

\begin{icmlauthorlist}
\icmlauthor{Samir Wadhwa}{uiuc_me}
\icmlauthor{Roy Dong}{uiuc_ece}
\end{icmlauthorlist}

\icmlaffiliation{uiuc_me}{Department of Mechanical Science and Engineering, University of Illinois, Urbana, Illinois}
\icmlaffiliation{uiuc_ece}{Department of Electrical and Computer Engineering, University of Illinois, Urbana, Illinois}

\icmlcorrespondingauthor{Samir Wadhwa}{samirw2@illinois.edu}

\icmlkeywords{causal inference, statistical learning theory}

\vskip 0.3in
]



\printAffiliationsAndNotice{}  

\begin{abstract}
Causal discovery methods seek to identify causal relations between random variables from purely observational data, as opposed to actively collected experimental data where an experimenter intervenes on a subset of correlates. One of the seminal works in this area is the Inferred Causation algorithm, which guarantees successful causal discovery under the assumption of a conditional independence (CI) oracle: an oracle that can states whether two random variables are conditionally independent given another set of random variables. Practical implementations of this algorithm incorporate statistical tests for conditional independence, in place of a CI oracle. In this paper, we analyze the sample complexity of causal discovery algorithms without a CI oracle: given a certain level of confidence, how many data points are needed for a causal discovery algorithm to identify a causal structure? Furthermore, our methods allow us to quantify the value of domain expertise in terms of data samples. Finally, we demonstrate the accuracy of these sample rates with numerical examples, and quantify the benefits of sparsity priors and known causal directions.
\end{abstract}


\section{Introduction}
\label{sec:intro}


Causal inference is growing in importance as data is used to make decisions. Understanding correlations is insufficient to understand the effect of an action or intervention taken by a decision-making agent. It would be fallacious reasoning for a decision-maker to to change an effect in hopes of inducing a change in an cause; however, distinguishing cause and effect is more difficult than identifying correlations.


Causal discovery methods seek to identify causal relations between random variables from purely observational data. 
Much of the data available for learning today was collected passively (i.e. without any experimenter intervening on any correlates), and, in such settings, we often wish to discover causal relations that were not known a priori. 
The earliest methods for causal discovery assumed that one had access to a conditional independence (CI) oracle: this oracle can reveal whether or not two random variables are independent conditioned on another set of random variables. With a CI oracle, the Inferred Causation (IC) algorithm is guaranteed to recover the causal dependencies up to fundamental ambiguities. (We will discuss this in greater detail in Section~\ref{sec:model}.)


In practice, when the IC algorithm is deployed in real-world settings, some changes are required. First, a conditional independence tester that checks for conditional independence from a finite number of observations typically replaces the CI oracle. Second, some amounts of a priori knowledge, which we refer to as {\em domain expertise}, may regularize the problem and improve its computational complexity and sample efficiency.

The contributions of this paper are as follows. 
\begin{itemize}
    \item 
To the best of our knowledge, this is the first paper that considers the sample complexity of causal discovery algorithms in the absence of a CI oracle. We combine recent results analyzing the sample complexity of conditional independence tests with an analysis of family-wise error rates in the repeated testing of causal inference algorithms to provide confidence bounds on the performance of causal discovery algorithms. This allows us to determine how many samples are needed to achieve a certain level of confidence in the recovery of the causal structure. 
    \item
Furthermore, we introduce the formalization of domain expertise as a partial CI oracle: domain expertise comes in the form of known results from particular conditional independence queries. This can encapsulate domain expertise in the form of a sparsity prior (e.g. an upper bound on the number of contextual variables that can render two random variables conditionally independent), a known causal direction, or knowledge of dependence or independence in a particular context. Our results allow us to quantify the value of domain expertise: how many additional samples are needed to achieve the same level of confidence in the absence of the provided domain expertise?
\end{itemize}

Our work is motivated by the observation that an understanding of the value of domain expertise allows one to more efficiently conduct experiments when some passively observed data is available. In the presence of some passively observed data, we can ask which domain expertise {\em would} best increase the confidence of causal discovery. 

The rest of the paper is organized as follows. In Section~\ref{sec:back}, we review related work. In Section~\ref{sec:model}, we outline the mathematical preliminaries required for our results. In Section~\ref{sec:theory}, we prove our key theorems, which give the sample complexity of causal discovery in the presence and absence of domain expertise. In Section~\ref{sec:computation}, we verify the rates with numerical simulations, and we end with closing remarks in Section~\ref{sec:conclusion}.


\section{Related Work}
\label{sec:back}


When events $A$ and $B$ co-occur, it is not obvious whether $A$ causes $B$, $B$ causes $A$, or neither. Classical approaches to handle the difficulty of causal inference either: 1) require an experimenter to actively intervene on certain factors, as is done in randomized control trials~\cite{Imbens:2015aa}, or 2) to a priori assume causal dependencies, as is typical in classical econometrics~\cite{Wooldridge:2019aa}. However, in the last several decades, there has been a growing interest in methods to identify causal dependencies from observational data.

One of the seminal works in this area is the development of the IC algorithm~\cite{Pearl:1995aa}. Building on a results in Bayesian networks, it identified conditions in which passive observations can reveal necessary links in a causal structure, and provided a computational method to identify said links. This original algorithm requires access to a CI oracle, which can exactly identify when two random variables are independent conditioned on some other set of contextual variables. 
Since then, the area of causal discovery has grown dramatically. 
One of the state-of-the-art methods for causal discovery is kernel-based methods. These methods use kernel-based methods for the conditional independence test in lieu of a CI oracle~\cite{Sun:2007aa,Zhang:2011aa,Mitrovic:2018aa}. These works are typically either empirically validated or provided asymptotic theoretical guarantees. In contrast, our work focuses on finite-sample results, using a conditional independence tester with known sample complexity.

Relatedly, \cite{Acharya:2018aa} seeks to identify which experiment should be done to best recover the causal graph, in either an adaptive or non-adaptive fashion. 
In contrast, our work attempts to quantify the value of potential interventions by identifying the comparable number of samples needed to achieve the same confidence level. 
\cite{Lee:2020aa} studies issues of identifiability when some variables are unobserved. 
Synergistic with our work is the work of~\cite{Jaber:2019aa}, which seeks to identify when domain expertise and observational data are sufficient to resolve ambiguities between observationally equivalent causal graphs.

Alternatively, some methods assume some known structure that regularizes the problem. Functional causal models assume that the causal dependencies have a known, parameterized functional form~\cite{Shimizu:2006aa,Hoyer:2009aa,Zhang:2009aa,Kumor:2020aa}. For an overview of functional causal model methods, we refer the reader to~\cite{Glymour:2019aa}.
In \cite{Greenewald:2019aa}, the authors assume the causal graph has a tree structure. 
Our work focuses on the settings where one is agnostic to the causal structures, but adds an assumption that the variables of interest are discrete.

To the best of our knowledge, this is one of the first works to analyze the sample complexity of causal discovery. The IC algorithm itself consists of repeated tests for conditional independence, and recent results on the sample complexity of these conditional independence tests are what enable the results in this paper. In particular, we build on the recent work of~\cite{Canonne:2018aa}, which provides sublinear conditional independence testing sample complexity.

Testing for conditional independence has a rich literature, and we can only provide a superficial summary here. The conditional independence testing problem has known hardness results when the variables are continuous (see, e.g.~\cite{Shah:2020aa}), and recent work has focused on identifying regularizing assumptions for the case where the variables are continuous~\cite{Matey-Neykov:2020aa}. Another interesting direction is to add assumptions on the computational complexity of the coupling between random variables, as is assumed in~\cite{Marx:2019aa}. 

Motivated by the hardness results for continuous variables, we focus on the setting where all the random variables are discrete. 
For discrete distributions, common methods include $G^2$ tests~\cite{Aliferis:2010aa,Schluter:2014aa} and conditional mutual information~\cite{Zhang:2010aa}. For $G^2$ tests, sample complexity bounds exist, but prior to the work of~\cite{Canonne:2018aa}, sublinear sample complexity results were not available. Our work uses these sublinear bounds in the analysis of the sample complexity of causal discovery methods.

\subsection{Notation}

We use the following notation throughout this paper. 
For a finite set $A$, we let $|A|$ denote the cardinality of $A$, and we let $2^A$ denote the powerset of $A$. Additionally, for any collection of sets $A_i$ across an index set $I$, we use $\sqcup_{i \in I} A_i$ to denote the disjoint union, whose elements are of the form $(i,x)$ with $x \in A_i$. 
For any event $A$, $P(A)$ will denote the probability of event $A$, with the underlying probability space implicitly understood. For any random variable $X$, we let $\mb{E}[X]$ denote the expectation of $X$, and let $\support(X)$ denote the support of $X$. We say $X \indep Y \vert Z$ to indicate that $X$ and $Y$ are conditionally independent given $Z$. 


\section{Mathematical Preliminaries}
\label{sec:model}

In this section, we formulate our problem and introduce the definitions and tools used for our results in Section~\ref{sec:theory}.

\subsection{Causal models}

First, we quickly introduce common definitions in causal inference. For a more detailed coverage of these concepts, we refer the reader to~\cite{Pearl:2013aa}.


\begin{definition}[Directed, acyclic graphs (DAGs)]
A {\bf directed graph} is a set of nodes $V$ and a set of edges $E \subseteq V \times V$. Throughout this paper, we assume $V$ is finite, and let $N = |V|$. 
For any node $i \in V$, we denote the {\bf parents} of $i$ as $\pa(i) = \{ j \in V : (j,i) \in E \}$. We iteratively define $n$-th generation ancestry as $\pa^n(i) = \{ j \in V : k \in \pa^{n-1}(i), (j,k) \in E \}$, with the base case $\pa^1 = \pa$. We define the {\bf ancestors} of a node $i \in V$ as $\anc(i) = \cup_{n = 1}^\infty \pa^n(i)$, and a directed graph is {\bf acyclic} if $i \notin \anc(i)$ for all $i \in V$.
\end{definition}

\begin{definition}[Markov compatibility]
Let $G = (V,E)$ be a DAG and let $X_i$ be a discrete random variable for each $i \in V$. We say $G$ and $X = (X_i)_{i \in V}$ are {\bf compatible} if the distribution of $X$ factorizes as:
\[
P(X = x) = \prod_{i \in V} P(X_i = x_i \vert (X_j)_{j \in \pa(i)} = (x_j)_{j \in \pa(i)})
\]
We also say $G$ {\bf represents} $X$, and refer to $G$ as a {\bf causal structure}.
\end{definition}


This formalization serves as the model for causality. 
When $G$ and $X$ are compatible, the causal interpretation is that if there is an edge from $i$ to $j$, then $X_i$ is a direct cause of $X_j$. 

One regularizing condition is one of {\em stability}: it states that every conditional independence in the variables $X$ arises necessarily from the causal structure $G$, and not some quirk of the parameterization. Intuitively, this means that a conditional independence holds if and only if some property of the causal structure is satisfied. 

\begin{definition}[Stability, faithfulness]
We say $G = (V,E)$ and $X$ satisfy the {\bf stability} condition if the following holds. For any $i \in V, j \in V$ and $A \subseteq V$, if $X_i \indep X_j \vert (X_k)_{k \in A}$, then $Y_i \indep Y_j \vert (Y_k)_{k \in A}$ for any other random variables $Y$ compatible with $G$. 
This condition is also sometimes referred to as {\bf faithfulness}.
\end{definition}

Finally, we should note that there are some fundamental ambiguities in causal discovery. In other words, there are some parts of the causal structure that can {\em never} be identified from passive observation alone. We formalize this by defining observationally equivalent DAGs.

\begin{definition}[Observational equivalence and patterns]
Two DAGs $G_1$ and $G_2$ are {\bf observationally equivalent} if for any random variables $X$ compatible with $G_1$, then $X$ is also compatible with $G_2$, and vice versa. Observational equivalence forms an equivalence class on the set of DAGs. 

These equivalence classes can be represented by {\bf patterns}, which are partially directed graphs (i.e. some edges are directed and some are undirected). Undirected edges will be denoted $\{i,j\} \in E$, and the interpretation is that either $(i,j) \in E$ or $(j,i) \in E$ for any DAG in the equivalence class.
\end{definition}

The IC algorithm is provided in Algorithm~\ref{alg:ic}. The IC algorithm provides a computational method by which to identify the equivalence class of DAGs which generated the random variables $X$, assuming access to a conditional independence oracle, defined here.


\begin{definition}[Conditional independence (CI) oracle] 
A CI oracle takes in any $i \in V, j \in V$ and $B \subseteq V$ and outputs $true$ if $X_i \indep X_j \vert (X_k)_{k \in B}$ and $false$ otherwise.
\end{definition}

\begin{algorithm}[tb]
   \caption{Inferred Causation (IC) algorithm}
   \label{alg:ic}
\begin{algorithmic}
   \STATE {\bfseries Input:} A collection of nodes $[N]$ and a CI oracle.
   \STATE {\bfseries Output:} An equivalence class of DAGs.
  \STATE {\em // Step 1: Use data to construct an undirected graph such that $i$ and $j$ are connected if and only if $X_i$ and $X_j$ are always conditionally dependent.}
   \STATE Initialize an empty graph $V = [N]$ and $E = \emptyset$.
   \FOR{$\{i,j\}$ such that $i \in V$, $j \in V \setminus \{i\}$}
     \STATE $found = false$
     \FOR{$B \subseteq V \setminus \{i,j\}$}
     \IF{$X_i \indep X_j \vert (X_k)_{k \in B}$}
       \STATE $found = true$, $S_{\{i,j\}} = B$, $break$
     \ENDIF
     \ENDFOR
     \STATE {\bfseries if} $found = false$, {\bfseries then} add $\{i,j\}$ to $E$
   \ENDFOR
   
  \STATE {\em // Step 2: Add directed edges based on $v$-structures. For more details on $v$-structures, see~\cite{Pearl:2013aa}.}
   \FOR{$k \in V$}
     \FOR{\{i,j\} such that $\{i,k\} \in E$, $\{j,k\} \in E$, $\{i,j\} \notin E$, $i \neq j$}
       \STATE {\bfseries if} $k \notin S_{\{i,j\}}$ {\bfseries then} add $(i,k)$ and $(j,k)$ to $E$
     \ENDFOR
   \ENDFOR
   \STATE {\em // Step 3: Orient undirected edges without introducing new $v$-structures or introducing a directed cycle.}
   \STATE Orient undirected edges according to the rules outlined in~\cite{Pearl:2013aa}.
   

\end{algorithmic}
\end{algorithm}

We note that only Step 1 of Algorithm~\ref{alg:ic} requires data; Steps 2 and 3 are post-processing and do not require any data. Furthermore, the IC algorithm is complete, in the sense of the following proposition.

\begin{proposition}[Completeness of the IC algorithm~\cite{Verma:1992aa,Meek:1995aa}]
\label{prop:ic_good}

Suppose $X$ and $G$ satisfy the stability criteria. Then, there is a unique equivalence class of causal structures consistent with $X$, and the IC algorithm recovers this equivalence class.

\end{proposition}

\subsection{Conditional independence testers}

As noted previously, it is uncommon in practice to have access to a CI oracle. In practice, we often use a finite number of data samples and a conditional independence tester to identify when there is a conditional independence or dependence. In this section, we outline some of the theoretical results in conditional independence testing, as recently discovered by~\cite{Canonne:2018aa}.

If the conditional dependencies can be arbitrarily small, then it will be difficult to detect from a finite number of samples. As such, we assume that dependencies cannot be arbitrarily small.

\begin{assumption}[Minimum level of dependence]
\label{ass:min_dep}
Let $\mc{P}_{X, Y \vert Z}$ denote the set of probability distributions such that $X \indep Y \vert Z$. For our unknown distribution $P$, we assume either $X \indep Y \vert Z$ or $d_{TV}(P, \mc{P}_{X, Y \vert Z}) := \inf_{Q \in \mc{P}_{X, Y \vert Z}} \sup_A |P(A) - Q(A)| > \eps$, for some known constant $\eps > 0$.
\end{assumption}


The algorithm itself is relatively intuitive. Suppose we observe several samples of random variables $(X,Y,Z)$ and wish to determine if $X \indep Y \vert Z$. Essentially, for each possible value of $Z$, we test if $X \indep Y \vert Z = z$, and we aggregate the results across all possible values of $z$ as a test for conditional independence. This is provided in Algorithm~\ref{alg:ci_test}. Note that this test actually requires a random number of samples: drawing $K$ from a Poisson distribution ensures that the number of elements in each bin is independent, which is essential for the proof.

\begin{algorithm}[tb]
   \caption{Testing for conditional independence}
   \label{alg:ci_test}
\begin{algorithmic}
   \STATE {\bfseries Input:} A sample generator of $(X,Y,Z)$ from a distribution $P$, an expected number of samples $m$, a threshold $\tau$, and an independence tester $\Phi$.
   \STATE {\bfseries Output:} $true/false$ if it is believed that $X \indep Y \vert Z$ or not.
   \STATE Draw samples $(X_i,Y_i,Z_i)_{i = 1}^K$ from $P$, where the number of samples $K \sim Poisson(m)$.
   \STATE Bin the data into multisets $S_z = \{ (X_i, Y_i) : Z_i = z \}$.
   \STATE $A = 0$
   \FOR{$z$ in the support of $Z$}
     \IF{$|S_z| \ge 4$}
       \STATE $A = A + |S_z| \Phi(S_z)$
     \ENDIF
   \ENDFOR
   \STATE Return $true$ if $A \le \tau$ and $false$ otherwise.

\end{algorithmic}
\end{algorithm}

We can guarantee the sample complexity of Algorithm~\ref{alg:ci_test} under some conditions, as outlined in the following proposition.

\begin{proposition}[Sample complexity of CI tests~\cite{Canonne:2018aa}]
\label{prop:ci_rates}

Suppose the independence tester $\Phi$ satisfies the following properties when given $K$ samples of $(X,Y)$ from distribution $P$:
\[
\mb{E}[\Phi] = \|P - P_X \otimes P_Y \|_2^2
\]
\[
\operatorname{Var}[\Phi] \le C \left( \mb{E}[\Phi]/K + 1/K^2\right)
\]
Here, $P_X$ and $P_Y$ denote the marginals of $P$ and $\otimes$ refers to the product measure, and $C$ is some constant.

Furthermore, let $M = |\support(Z)|$, and suppose for some constant $\beta$:
\begin{equation}
    \label{eq:def_m}
m = \beta \max \left(
\frac{M^{1/2}}{\eps^2}, \min 
\left(
\frac{M^{7/8}}{\eps},
\frac{M^{6/7}}{\eps^{8/7}}
\right)
\right)
\end{equation}

In addition, let $\gamma = 1 - 5/2e$ and $\eps' = \eps / \sqrt{|\support(X)||\support(Y)|}$ (where $\eps$ is defined in Assumption~\ref{ass:min_dep}). Set $\tau = \frac{\gamma}{2} \min \left( m (\eps')^2,\frac{(m\eps')^4}{M^3} \right)$. 
Then, Algorithm~\ref{alg:ci_test} satisfies the following properties.

When $X \indep Y \vert Z$, there exists some constant $C' > 0$ such that:
\begin{equation}
    \label{eq:ind_bound}
P(false) \le \frac{C'}{ \beta^2 \gamma^2}
\end{equation}

When $d_{TV}(P, \mc{P}_{X, Y \vert Z}) > \eps$, for the same constant $C'$:
\begin{equation}
    \label{eq:dep_bound}
P(true) \le C' \left(\frac{64}{\beta^2 \gamma^2} + \frac{8}{\beta \gamma \sqrt{\min(M,m)}}\right)
\end{equation}

\end{proposition}

In Section~\ref{sec:theory}, we will use Proposition~\ref{prop:ci_rates} to determine how many samples are needed to achieve a particular confidence level. The parameter of importance is $\beta$, which scales the expected number of samples $m$ and changes the confidence level. 
Additionally, we emphasize the dependence on $M$, the cardinality of the support of $Z$, which changes between iterations of the IC algorithm.

\subsection{Family-wise error rates and Bonferroni correction}

In the absence of a CI oracle, Step 1 of the IC algorithm requires one to conduct a family of conditional independence tests. In this paper, we prove confidence levels for the successful recovery of the underlying causal structure's pattern. We do so by noting when our conditional independence tests function as a CI oracle. In other words, when is every conditional independence test correct?

The Bonferroni correction is a union bound on the confidence of a family of tests~\cite{Miller:1981aa}. 
We note that the Bonferroni correction method does not require the tests to be independent, as it simply relies on the union bound. Other methods for controlling the family-wise error rate often have restrictions on the dependence between the tests (e.g. \u{S}id\'{a}k correction or Hochberg's procedure~\cite{Miller:1981aa}), which is difficult to verify for the conditional independence tests.

\begin{proposition}[Bonferroni correction~\cite{Miller:1981aa}]
\label{prop:bonferroni}
Consider a family of conditional independence tests, $\{CI_i\}_i$, where each test $CI_i$ has a confidence level of $1 - \alpha_i$. Then, the family of tests has a confidence level of $1 - \sum_i \alpha_i$.
\end{proposition}


We emphasize that when the same data is used for each conditional independence test, then the outcomes of the tests will necessarily be coupled. In our setting, we assume that all data is used in every conditional independence test. An alternative method would be to use each data point in exactly one conditional independence test. This would ensure that the tests are independent, but the rate at which data gets `used up' with this method ultimately provides worse sample complexity results, due to the large number of conditional independence tests in the IC algorithm.


Proposition~\ref{prop:ic_good} implies that, should every conditional independence test yield the true result, then we will recover the causal structure's pattern. 


\section{Sample Complexity of Causal Discovery Algorithms}
\label{sec:theory}


In this section, we derive the rates for sample complexity of causal discovery algorithms.

Recall that $V$, the nodes of our causal graph, are an index set for a collection of random variables, $(X_i)_{i \in V}$. Let us first introduce notation to index the conditional independence tests in the IC algorithm (Algorithm~\ref{alg:ic}). Let:
\begin{equation}
    \label{eq:T_def}
T = \bigsqcup_{\{i,j\} \subseteq V} 2^{V \setminus \{i,j\}}
\end{equation}
Note that each element of $T$ is of the form $(\{i,j\}, B)$, where $B \subseteq V \setminus \{i,j\}$. Thus, each element of $T$ corresponds to a conditional independence test which may need to be run.\footnote{Not every test may need to be run, based on the outcome of previous tests, but for our analysis in this paper, we assume the worst-case and suppose that every test must be run.}

First, we adapt the results in Proposition~\ref{alg:ci_test} to serve our purposes: for one fixed conditional independence test, we identify how many samples are needed in expectation to achieve a certain confidence level.

\begin{lemma}
\label{lem:one_CI_test}
Fix a conditional independence test $(\{i,j\}, B) \in T$ and a desired confidence level $1 - \alpha$. Suppose the expected number of samples $m$ is greater than the cardinality of the support of $B$. Let $M = |\support(B)|$.

Then, the desired confidence level can be achieved with $Poisson(m)$ samples, where:
\[
m(M,\alpha) = 
\begin{cases}
\frac{16 C'}{\alpha\gamma (\eps')^2} & \text{if } M \le (\eps')^{-8/3} \\
\frac{16 C'}{\alpha\gamma \eps'} M^{3/8} & \text{if } (\eps')^{-8/3} < M \le (\eps')^{-8} \\
\frac{16 C'}{\alpha\gamma (\eps')^{8/7}} M^{5/14} & \text{if } M > (\eps')^{-8} \\
\end{cases}
\]
Here, $C'$, $\gamma$, and $\eps'$ are as defined in Proposition~\ref{prop:ci_rates}.

\end{lemma}

\begin{proof}
Using Proposition~\ref{prop:ci_rates}, it suffices to bound both the false negative rate in Equation~\eqref{eq:ind_bound} and the false positive rate in Equation~\eqref{eq:dep_bound}. We can see that the false positive rate is larger than the false negative rate. Since we are assuming $m$ is greater than $M$, it suffices to choose $\beta$ such that $64 C'/\beta^2 \gamma^2 \le \alpha/2$ and $8C'/\beta\gamma\sqrt{M} \le \alpha/2$. Additionally, in the high sample regime, the latter inequality will imply the former; rearranging gives the condition:
\[
\beta \ge \frac{16 C'}{\alpha \gamma \sqrt{|\support(Z)|}}
\]
Next, noting the 3 regimes in the definition of $m$ in Equation~\eqref{eq:def_m} yields the desired result.
\end{proof}

Lemma~\ref{lem:one_CI_test} tells us how many samples are needed for one test to achieve a particular confidence level. Combining Lemma~\ref{lem:one_CI_test} with the Bonferroni correction method in Proposition~\ref{prop:bonferroni}, if each test $t \in T$ has confidence level $1 - \alpha_t$, then the family of tests will have confidence $1 - \sum_{t \in T} \alpha_t$. By Proposition~\ref{prop:ic_good}, this means we will recover the equivalence class of causal structures with probability at least $1 - \sum_{t \in T} \alpha_t$. 

However, if we have a target confidence level $1 - \alpha$, we must set $\alpha_t$ for each $t \in T$. We can think of it as an $\alpha$ budget of uncertainty which must be divided between the tests. For each $t = (\{i,j\},B) \in T$, let $M_t = |\support(B)|$. Then, for any vector of positive constants $(\alpha_t)_{t \in T}$, we can define the cost $\max_{t \in T} m(M_t,\alpha_t)$ as the expected number of samples required for every test $t \in T$ to have confidence $\alpha_t$.

The following optimization provides the lowest expected number of samples for a target confidence level $1 - \alpha$:
\begin{equation}
\label{eq:opt_alpha}
\begin{aligned}
\min_{(\alpha_t)_t} \quad &
\max_{s \in T} m(M_s,\alpha_s) \\
\textrm{s.t.} \quad &
\sum_{t \in T} \alpha_t = \alpha \\
& \alpha_t \ge 0 \text{ for all } t \in T   \\
\end{aligned}
\end{equation}

We can characterize solutions to this optimization with the following lemma.

\begin{lemma}
\label{lem:m_equal}
Let $(\alpha_t^*)_t$ denote an optimal solution to Equation~\eqref{eq:opt_alpha}. Then $m(M_t, \alpha_t^*) = m(M_s, \alpha_s^*)$ for any $s \in T$ and $t \in T$.
\end{lemma}

\begin{proof}
Note that $m(M,\alpha)$ is continuous and strictly decreasing in $\alpha$ for any fixed $M$. Let $T^* \subseteq T$ denote the set of tests $t^*$ such that $m(M_{t^*},\alpha_{t^*}^*) = \max_t m(M_t, \alpha_t^*)$. Suppose for the sake of contradiction that there exists an $s \in T$ such that $m(M_s, \alpha_s^*) < m(M_{t^*},\alpha_{t^*}^*)$. For $\delta > 0$, we can define $(\tilde \alpha_t)_t$:
\[
\tilde \alpha_t =
\begin{cases}
\alpha_t^* + \delta / |T^*| & \text{if } t \in T^* \\
\alpha_t^* - \delta & \text{if } t = s \\
\alpha_t^* & \text{otherwise}
\end{cases}
\]
For $\delta$ sufficiently small, $\max_{t \in T} m(M_t, \tilde \alpha_t) < \max_{t \in T} m(M_t, \alpha_t^*)$, contradicting the optimality of $(\alpha_t^*)_t$.
\end{proof}

This characterization allows us to state our theorem which gives the sample complexity of the IC algorithm in the absence of a CI oracle.

\begin{theorem}[Sample complexity of the IC algorithm]
\label{th:sample_complexity}
As before, let $N = |V|$ denote the number of random variables considered. Let $R_1 = (\eps')^{-8/3}$ and $R_2 = (\eps')^{-8}$. Fix any $\alpha \in (0,1)$, and let $\alpha_0^*$ solve:
\begin{equation}
\label{eq:alpha0star_def}
\begin{split}
\alpha_0^* \bigg( \sum_{t \in T : M_t \le \lfloor R_1 \rfloor } 1 + 
\sum_{t \in T : \lfloor R_1 + 1 \rfloor \le M_t \le \lfloor R_2 \rfloor} 
\eps' M_t^{3/8} + \\
\sum_{t \in T : M_t \ge \lfloor R_2 + 1 \rfloor} 
 (\eps')^{6/7} M_t^{5/14} \bigg) = \alpha
\end{split}
\end{equation}
Then, given $Poisson(16C'/\alpha_0^*\gamma(\eps')^2)$ samples, Algorithm~\ref{alg:ic}, using Algorithm~\ref{alg:ci_test} as the conditional independence tester, recovers the correct equivalence class of causal models with probability greater than $1 - \alpha$. Here, $\gamma$ and $\eps'$ are as defined in Proposition~\ref{prop:ci_rates}.
\end{theorem}

\begin{proof}
Let $(\alpha_t^*)_t$ denote an optimal solution to Equation~\eqref{eq:opt_alpha}. By Lemma~\ref{lem:m_equal}, it suffices to calculate $\alpha_0^*$, where $m(0,\alpha_0^*)$ equals the optimal value of Equation~\eqref{eq:opt_alpha}.\footnote{Here, we adopt a slight abuse of notation, noting that the value of $\alpha_t^*$ depends only on $M_t = |\support(B)|$ where $t = (\{i,j\},B)$.} Note that $m(0,\alpha_0) = 16C'/\alpha_0\gamma(\eps')^2$. Then, for any other $M$, we see the following holds by Lemma~\ref{lem:one_CI_test}:
\[
\alpha_M^*=
\begin{cases}
\alpha_0^* & \text{if } M \le (\eps')^{-8/3} \\
\alpha_0^* \eps' M^{3/8} & \text{if } (\eps')^{-8/3} < M \le (\eps')^{-8} \\
\alpha_0^* (\eps')^{6/7} M^{5/14} & \text{if } M > (\eps')^{-8}
\end{cases}
\]
The desired result follows by summing these terms across the total budget constraint on $\alpha$ in Equation~\eqref{eq:opt_alpha}.
\end{proof}

This bound becomes easier to interpret in the case where the variables are all have the same cardinality.

\begin{corollary}
\label{corr:same_supp_1}
Suppose, every random variable $X_i$ has the same number of possible values $\ell = |\support(X_i)|$. Then a confidence level of $1-\alpha$ is achieved with $Poisson(m)$ samples, where:
\[
m \ge \frac{16 C'}{\alpha\gamma (\eps')^2} {N \choose 2} (1+\ell^{3/8})^{N-2}
\]
\end{corollary}
\begin{proof}
This follows by bounding the term in the parentheses in Equation~\eqref{eq:alpha0star_def}. 
Note that the terms within the sum are all upper-bounded by $\ell^{3|B|/8}$, where the test $t = (\{i,j\},B)$ and $|B|$ is the number of variables in the set $B$.
Thus, we have the bound $\alpha \le \alpha_0^* \sum_{t \in T} \ell^{3|B|/8}$. 
The rest follows from counting the number of elements in $T$: $\alpha_0^* \sum_{t \in T} \ell^{3|B|/8} = \alpha_0^* {N \choose 2} (1 +\ell^{3/8})^{N-2}$. 
Plugging this bound into $m(0,\alpha_0^*)$ yields the desired result.
\end{proof}


\subsection{The effect of domain expertise}

Now, we formalize domain expertise as a partial CI oracle. Intuitively, domain expertise allows us to know the outcome of some subset of conditional independence tests without using any samples. Thus, these terms do not enter into the Bonferroni correction evaluations of the previous section.

\begin{definition}[Partial CI oracle]
A {\bf partial CI oracle} is specified by some subset of $S \subseteq T$ (where $T$ is as defined in Equation~\eqref{eq:T_def}). For any $(\{i,j\},B) \in S$, the partial CI oracle outputs $true$ if $X_i \indep X_j \vert (X_k)_{k \in B}$ and $false$ otherwise.
\end{definition}

In this formalism, it is straightforward to modify the sample complexity results from the previous section accordingly.

\begin{corollary}[Sample complexity with domain expertise]
Suppose we have access to a partial CI oracle, which can respond to queries $S \subseteq T$. Let $\alpha_0^*$ solve Equation~\eqref{eq:alpha0star_def}, only with the sums across $t \in T \setminus S$ instead of $t \in T$.

Then, given $Poisson(16C'/\alpha_0^*\gamma(\eps')^2)$ samples, Algorithm~\ref{alg:ic}, using Algorithm~\ref{alg:ci_test} as the conditional independence tester, recovers the correct equivalence class of causal models with probability greater than $1 - \alpha$.
\end{corollary}


This corollary holds for general partial CI oracles. When more structure is provided, we can more explicitly calculate the bounds. For example, the PC algorithm is based on a priori knowledge of the sparsity of the graph~\cite{Spirtes:2001aa}. If a causal direction is known, then this also corresponds to a set of tests which we do not need to run. This is formalized in the following corollaries.


\begin{corollary}[Sample complexity under a sparsity prior]
\label{corr:pc}
Suppose every random variable $X_i$ has the same number of possible values $\ell = |\support(X_i)|$, and we know a priori that any context $(X_k)_{k \in B}$ that renders two random variables $X_i$ and $X_j$ conditionally independent will have at most $|B| \le R$ random variables. Then a confidence level of $1-\alpha$ is achieved with $m$ such that:
\[
m \ge \frac{16 C'}{\alpha\gamma (\eps')^2} {N \choose 2} \sum_{k = 0}^R {{N - 2} \choose k} \ell^{3k/8}
\]
\end{corollary}

\begin{proof}
This follows similarly to Corollary~\ref{corr:same_supp_1}, where we take the oracle as defined on $S = \{ (\{i,j\},B) \in T : |B| \le R \}$.
\end{proof}

We observe that the growth rate of the sample complexity for the PC algorithm is much smaller than the exponential growth rate of sample complexity for the IC algorithm. We visualize this in Figure~\ref{fig:ICvsPC}. Thus, our bounds allow us to quantify the value of the sparsity prior in terms of the number of samples saved. Intuitively, as the number of nodes grows, this sparsity prior becomes more and more beneficial.

\begin{figure}[!h]
\vskip 0.2in
\begin{center}
\centerline{\includegraphics[width=\columnwidth]{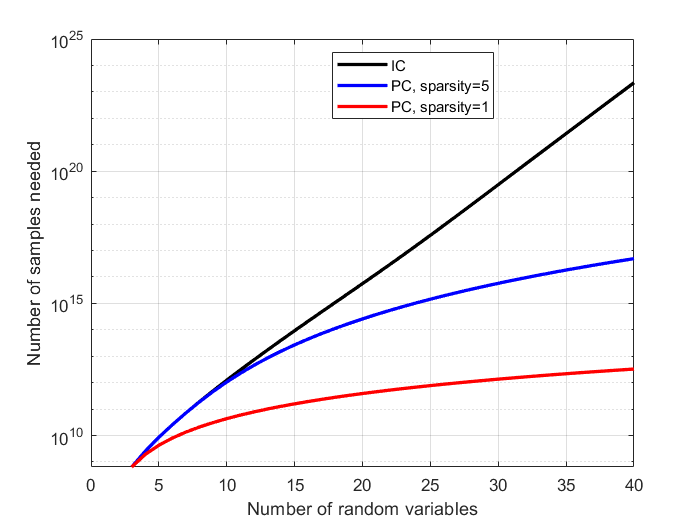}}
\caption{The number of samples needed to guarantee a confidence level of $1 - \alpha = 0.95$ with $\ell = 2$ and $\eps = 0.1$.}
    \label{fig:ICvsPC}
\end{center}
\vskip -0.2in
\end{figure}

\begin{corollary}[Sample complexity with a known causal direction]
Suppose every random variable $X_i$ has the same number of possible values $\ell = |\support(X_i)|$. Furthermore, suppose we know a priori the following set of causal dependencies: $D = \{(i_k,j_k)\}_k$ where $(i_k,j_k)$ is an edge in our causal graph for each $k$. Then a confidence level of $1-\alpha$ is achieved with $m$ such that:
\[
m \ge \frac{16 C'}{\alpha\gamma (\eps')^2} \left({N \choose 2}-|D|\right) (1+\ell^{3/8})^{N-2}
\]
\end{corollary}


We note a few things about this model for domain expertise and our analysis. It assumes that the partial CI oracle merely replaces the conditional independence tests in Algorithm~\ref{alg:ic}. It does not assume that any re-ordering of tests is done, in light of the results of the oracle queries. One interpretation of this is that one does not assume the result is known until {\em after} the oracle is queried. If we know a priori that a test result will successfully identify a conditional independence, this will correspond to a loop break in Algorithm~\ref{alg:ic}, and some tests in $T$ will not need to be conducted. Furthermore, we may be able to re-order the tests to reduce the tests that should be conducted. 

The reasons for analyzing sample complexity from an ex-ante perspective are two-fold. First, it is easier to analyze, and provides a worst-case bound in the ex-post setting as well. Additionally, this formalization is due in part to our motivating application. 
This analysis can identify which experiments we {\em should} conduct, from a data-efficiency perspective. In practical settings, the decision to conduct an experiment must be made without knowing what the outcome of the experiment will be.


A limitation of this definition is it only captures the effect of domain expertise on sample complexity. Another fashion in which this knowledge can be useful is to distinguish observationally equivalent causal graphs: some a priori knowledge could identify the true causal graph from the equivalence class of observationally equivalent graphs. This is not captured in our formalism.


\section{Numerical Validation}
\label{sec:computation}


In this section we compare our bounds for sample complexity against results obtained by running the IC and PC algorithm on simulated data for causal graphs with different number of nodes. 
For the simulations in this section, we take the causal graph of $N$ nodes to be the graph where $X_1, X_2,.., X_{N-1}$ are independent and are direct causes of $X_N$. For the simulation, $P(X_i = 0) = 0.6$ for $i = 1,2,..,N-1$ and $X_N = X_1 \lor X_2 \lor ... X_{N-1}$. The causal graph is shown in Figure~\ref{fig:causal_graph}.

\begin{figure}[h]
\vskip 0.2in
\begin{center}
\centerline{\includegraphics[width=0.75\columnwidth]{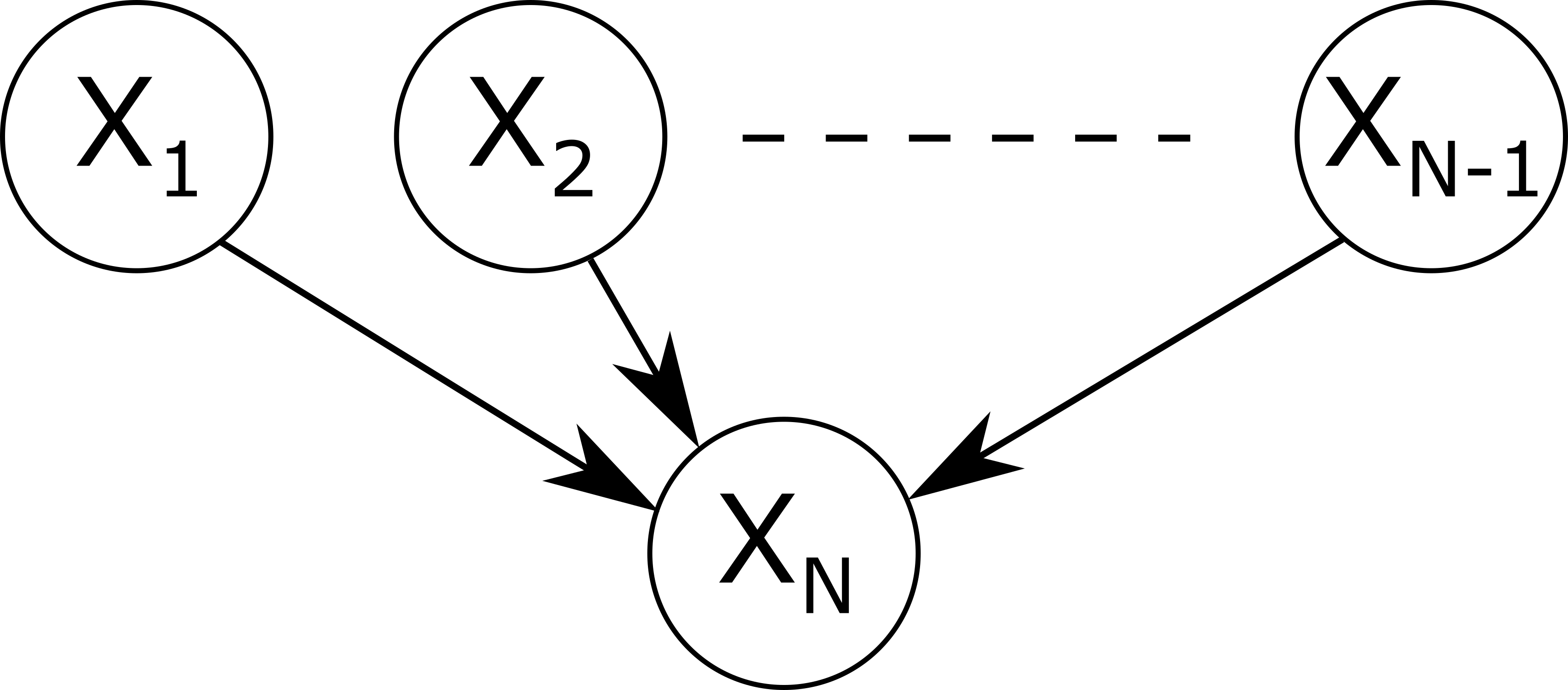}}
\caption{Causal graph with $N$ nodes used in simulations.}
    \label{fig:causal_graph}
\end{center}
\vskip -0.2in
\end{figure}

\subsection{Results for IC Algorithm}

We ran the IC algorithm on data generated according to the causal graph in Figure~\ref{fig:causal_graph}. We compare the empirical error rate from our numerical simulations with the theoretical error rate determined by Theorem~\ref{th:sample_complexity} in Figure~\ref{fig:results_IC}. We note that our results were up to a multiplicative constant, and, when these constants were explicitly calculated, the number of samples proved to be quite conservative. There are a few sources of conservativeness in our analysis. First, the union bound used to control the family-wise error rate is conservative. Second, our analysis can only guarantee confidence when every conditional independence test is accurate; however, in practice, we may still recover the correct causal graph even when some tests fail. For example, when $N = 3$, if the $X_1 \indep X_2$ test evaluates to $false$ but the independence test $X_1 \indep X_2 \vert X_3$ evaluates to $true$, the true causal graph will be recovered.
However, this is difficult to use in any confidence level guarantees: it is difficult to determine which tests can fail while guaranteeing recovery of the true causal graph, especially without knowledge of the true causal structure. We note that the sample complexity rates are relatively representative of the empirical rates up to a scaling constant, which suggests that these bounds, though conservative, may still provide a useful quantification of the benefits of domain expertise.

\begin{figure}[!ht]
\vskip 0.2in
\begin{center}
\centerline{\includegraphics[width=\columnwidth]{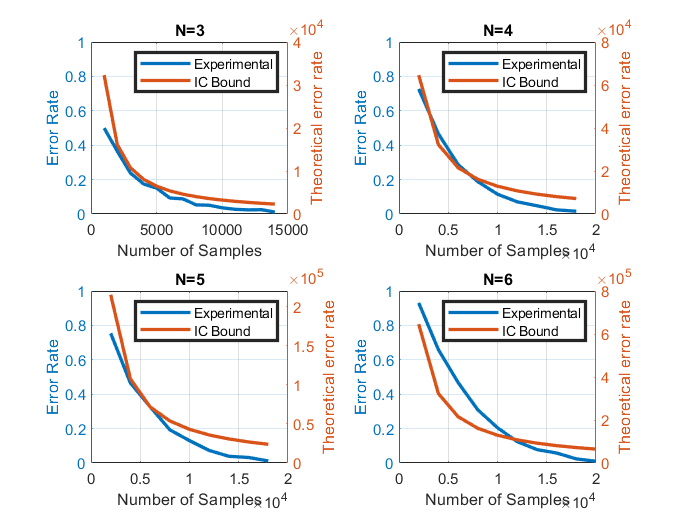}}
\caption{A comparison of the empirical error rate for our causal model using the IC algorithm, compared with our theoretical guarantees. The empirical rate is calculated based on 1000 trials. $N$ represents the number of nodes in the causal graph model of Figure~\ref{fig:causal_graph}.}
    \label{fig:results_IC}
\end{center}
\vskip -0.2in
\end{figure}

\subsection{Results for PC Algorithm}

We conducted a similar comparison for the PC algorithm in Figure~\ref{fig:results_PC}, compared with the theoretical rates in Corollary~\ref{corr:pc}. Similar to the IC case, we see that the rates seem accurate up to a multiplicative constant. Additionally, the multiplicative constant is far better due to the sparsity prior reducing the error introduced by the union bound.

\begin{figure}[!ht]
\vskip 0.2in
\begin{center}
\centerline{\includegraphics[width=\columnwidth]{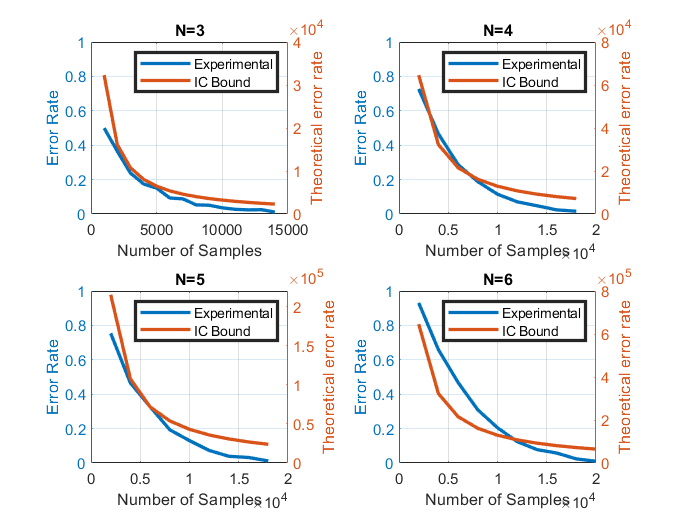}}
\caption{A comparison of the empirical error rate for our causal model using the PC algorithm with a sparsity prior of $R = 1$, compared with our theoretical guarantees. The empirical rate is calculated based on 1000 trials. $N$ represents the number of nodes in the causal graph model of Figure~\ref{fig:causal_graph}.}
    \label{fig:results_PC}
\end{center}
\vskip -0.2in
\end{figure}


\section{Closing Remarks}
\label{sec:conclusion}

In this paper, we outlined the sample complexity of the IC algorithm in the absence of a CI oracle. When domain expertise or other a priori knowledge can be modeled with a partial CI oracle, we have specified the sample complexity of the IC algorithm with this side information. These results were derived by combining recent bounds on the sample complexity of conditional independence testing, an analysis of the graph traversal in the IC algorithm, and family-wise error rate correction methods. These sample complexity rates can serve as a guiding principle for experiment design when passively observed data is available: an experiment can be modeled as knowing the independence (or lack thereof) of two random variables conditioned on a particular context. Thus, our results provide a quantifiable metric to evaluate which experiments (if conducted) can best improve the confidence of causal discovery tests, in the presence of both passively and actively collected data.




\pagebreak

\bibliography{DONG_ROY-refs}
\bibliographystyle{icml2021}


\end{document}